\documentclass[letterpaper, 10pt, conference]{ieeeconf}

\usepackage{graphics}
\usepackage{epsfig}

\usepackage{cite}
\usepackage[dvipsnames, svgnames]{xcolor}
\usepackage{xurl}
\usepackage{hyperref}
\hypersetup{
    colorlinks,
    linkcolor={red!50!black},
    citecolor={blue!50!black},
    urlcolor={blue!80!black}
}
\usepackage{bm}
\usepackage{breqn}
\usepackage{amssymb}
\usepackage{tabularx}
\usepackage{booktabs}
\usepackage{dcolumn}
\newcolumntype{d}[1]{D..{#1}}
\newcolumntype{C}{>{\centering\arraybackslash}X}
\usepackage{soul}
\usepackage[linesnumbered,ruled,vlined]{algorithm2e}

\usepackage{amsthm}

\newcommand{\reals}{\mathbb{R}}

\newcommand{\R}{\reals}

\newcommand{\nonnegintegers}{\mathbb{Z}_0}
\newcommand{\posintegers}{\mathbb{Z}_+}

\renewcommand{\S}{\mathbb{S}}

\usepackage{graphicx}  
\usepackage{array}     

\newcommand{\pd}{\S_{++}}

\newcommand{\Rcal}{\mathcal{R}}

\newcommand{\Tcal}{\mathcal{T}}
\newcommand{\Ucal}{\mathcal{U}}
\newcommand{\Xcal}{\mathcal{X}}

\newcommand{\Zcal}{\mathcal{Z}}

\newcommand{\eqn}[1]{\begin{align} #1 \end{align}}
\newcommand{\eqnN}[1]{\begin{align*} #1 \end{align*}}

\newcommand{\norm}[1]{\left\Vert #1 \right \Vert}

\theoremstyle{plain}

\theoremstyle{remark}

\definecolor{yellow}{cmyk}{0.0,0.10,0.95,0.0}
\definecolor{pred}{cmyk}{0,0.8,0.70,0.0}
\definecolor{bluedefined}{cmyk}{0.46, 0.10, 0, 0.0}

\IEEEoverridecommandlockouts  
\usepackage[english]{babel}
\usepackage{footnotehyper}

\usepackage{hyperref,tablefootnote,footnotehyper}
\usepackage{amsmath,booktabs,mwe,pdflscape}
\hypersetup{colorlinks}

\usepackage{leftindex} 
\usepackage{graphicx} 
\usepackage{amsmath}
\usepackage{amssymb}
\usepackage{amsthm}
\usepackage{amsfonts}
\usepackage{float}
\usepackage{mathrsfs}
\usepackage{booktabs}
\usepackage{array}
\usepackage[nodisplayskipstretch]{setspace}
\makeatletter
\def\BState{\State\hskip-\ALG@thistlm}
\makeatother
\usepackage[capitalize,nameinlink,noabbrev]{cleveref} 
\theoremstyle{plain}

\newtheorem{thm}{Theorem}

\theoremstyle{definition}
\newtheorem{prob}{Problem}
\newtheorem{assumption}{Assumption}

\newcolumntype{M}[1]{>{\centering\arraybackslash}m{#1}}
\usepackage{pifont}
\usepackage{multirow}
\usepackage{tabularray}
\usepackage{cite}
\usepackage{algpseudocode}

\newcommand{\RNum}[1]{\uppercase\expandafter{\romannumeral #1\relax}}
\usepackage{bigints}

\newcommand{\eware}{\textbf{\texttt{eware}}}
\newcommand{\mesch}{\textbf{\texttt{meSch}}}
\newcommand{\gware}{\textbf{\texttt{gware}}}
\algrenewcommand\textproc{}

\usepackage[font=small,skip=10pt]{caption}

\algrenewcommand\algorithmicrequire{\textbf{Input:}}
\algrenewcommand\algorithmicensure{\textbf{Output:}}

\hypersetup{
  colorlinks,
  citecolor=teal, 
  linkcolor=teal,
  urlcolor=purple}

\author{Kaleb Ben Naveed$^{1}$, An Dang$^{1}$, Rahul Kumar$^{1}$, and Dimitra Panagou$^{1,2}$%
\thanks{$^{*}$The authors would like to acknowledge the support of the National Science Foundation (NSF) under grant no. 2223845 and grant no. 1942907.}
\thanks{$^{1}$Department of Robotics, University of Michigan, Ann Arbor, MI, 48109 USA. 
{\tt\small \{kbnaveed@umich.edu\}}}
\thanks{$^{2}$Department of Aerospace Engineering, University of Michigan, Ann Arbor, MI, 48109 USA. }}%

\title{\LARGE \bf
\mesch: Multi-Agent Energy-Aware Scheduling for Task Persistence
} 

\def\arraystretch{1.2}

\newcommand\footnoteref[1]{\protected@xdef\@thefnmark{\ref{#1}}\@footnotemark}

\begin{document}

\maketitle
\thispagestyle{empty}
\pagestyle{empty}

\begin{abstract}
This paper develops a scheduling protocol for a team of autonomous robots that operate on long-term persistent tasks. The proposed framework, called \mesch{}, accounts for the limited battery capacity of the robots and ensures that the robots return to charge their batteries one at a time at the single charging station. The protocol is applicable to general nonlinear robot models under certain assumptions, does not require robots to be deployed at different times, and can handle robots with different discharge rates. We further consider the case when the charging station is mobile and its state information is subject to uncertainty. The feasibility of the algorithm in terms of ensuring persistent charging is given under certain assumptions, while the efficacy of \mesch{} is validated through simulation and hardware experiments. 
\href{https://github.com/kalebbennaveed/meSch.git}{[Code]}\footnote{Codebase: https://github.com/kalebbennaveed/meSch.git}\href{https://youtu.be/wgH-KgNGJgw}{[Video]}\footnote{Video: https://youtu.be/wgH-KgNGJgw}
\end{abstract}



\section{Introduction}
Autonomous robots are extensively used for persistent missions over long time horizons such as search and rescue operations \cite{search_rescue2}, water body exploration \cite{water_exploration2}, ocean current characterization \cite{gulf2}, and gas leakage inspection \cite{gas_leakage}. 

Trajectory planning for persistent missions requires not only optimizing high-level objectives but also ensuring \textbf{\textit{task persistence}}, realized in the sense that robots must be able to return for recharging as needed. When there is only one available charging station, the robots should coordinate their schedules to ensure the charging station's exclusive use. Additionally, in persistent monitoring scenarios, maximizing the number of robots actively monitoring the environment is desirable, and ensuring that only one robot returns for recharging promotes this behavior. In this paper, we consider the above challenges, including the case where the charging station is mobile, and propose a solution framework \mesch{}.


Recent work on task persistence has explored strategies for single-agent\cite{single_agent_1, single_agent_3, single_agent_5, single_agent_6, naveed2023eclares} as well as multi-agent scenarios\cite{dedicated_01, dedicated_02, dedicated_03, singall_for_all_1, periodic_charging, static_charging_bipartite, static_placing_charging,seewald2024energyaware, static_Arp, bentz2018complete, persis_Fouad, multirobot_game_theoretic, pre-plan_1, pre-plan_2, continous_comms_1, cooperative_1, cooperative_2 }. In the realm of multi-agent task persistence, existing methods explore scenarios with both static \cite{dedicated_01, dedicated_02, dedicated_03, singall_for_all_1,periodic_charging, static_charging_bipartite,static_placing_charging,seewald2024energyaware,static_Arp, bentz2018complete, persis_Fouad, multirobot_game_theoretic} and mobile charging stations \cite{multirobot_game_theoretic, pre-plan_1, pre-plan_2, continous_comms_1, cooperative_1,cooperative_2, MOBILE_RAL_2022}. Most of the existing work on multi-agent task persistence with static charging stations assumes either a dedicated charging station for each agent \cite{dedicated_01, dedicated_02, dedicated_03, multirobot_game_theoretic} or a single charging station that can cater to multiple agents simultaneously \cite{singall_for_all_1, periodic_charging}. Additionally, some studies address the scenarios where the number of charging stations available is fewer than the number of robots \cite{static_charging_bipartite, static_placing_charging, seewald2024energyaware}. They do this by either altering the nominal paths to visit charging stations along the way \cite{static_charging_bipartite} or strategically placing the shared charging resources \cite{static_placing_charging}, or by imposing constraints to ensure that the number of returning robots matches the number of available charging stations \cite{seewald2024energyaware}. In this paper, we explore a scenario where a team of robots exclusively shares a single charging station. 

The closest previous works to this setup are \cite{bentz2018complete, persis_Fouad}. \cite{bentz2018complete} ensured exclusive use of the charging station by deploying the robots at different times. Conversely, \cite{persis_Fouad} ensured exclusive use by manipulating the minimum allowed SoC (State-of-Charge) of each robot's battery using control barrier functions (CBF) \cite{CBF_TAC}. While the CBF-based approach developed in \cite{persis_Fouad} is computationally efficient, it is specifically tailored to the single-integrator model. In this paper, we propose an online scheduling method that ensures the exclusive use of the single charging station. The proposed method does not require robots to be deployed at different times or with different SoC capacities and applies to nonlinear robot models. Additionally, unlike \cite{bentz2018complete, persis_Fouad}, we also address the scenario when the charging station is mobile. 

Much of the literature addressing mobile charging stations considers that one of the robots in the team serves as a mobile charging station. Some approaches pre-plan the robots' paths for the entire mission to determine appropriate rendezvous points where robots will meet for recharging during the mission \cite{pre-plan_1, pre-plan_2}. Others assume continuous communication between the robots to determine the rendezvous points online whenever one of the robots needs to recharge \cite{continous_comms_1}. The other direction of work deploys separate mobile charging robot(s), and plans their paths to recharge the robots involved in the mission along their nominal paths \cite{ cooperative_1, cooperative_2, MOBILE_RAL_2022}.

In this paper, we consider a scenario where one of the robots in the persistent mission serves as a mobile charging station. Additionally, we address the case where robots have uncertain information about the mobile charging station' position due to erroneous state estimation or external disturbances. With this in mind, we devise a method to robustly estimate rendezvous points online while accounting for this uncertainty. The paper is organized as follows: \cref{prob_form} formulates the problem statement, \cref{meSch} presents the proposed solution, and \cref{sec:results} discuss results.




\section{Problem Formulation}
\label{prob_form}



\begin{figure*}[t]
  \centering
  \includegraphics[width=1.95\columnwidth]{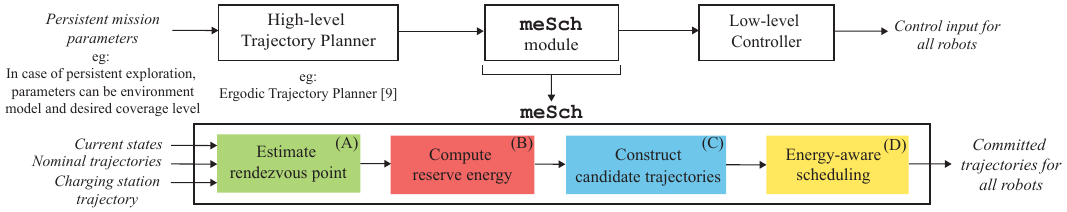}
  \caption{\mesch: Multi-Agent Energy-Aware Scheduling }
  \vspace{-17pt}
  \label{fig:mesch_overview}
\end{figure*}
\subsection{Notation}
Let $\nonnegintegers = \{ 0, 1, 2, ... \}$ and $\posintegers = \{1, 2, 3, ... \}$. Let $\mathbb{R}$, $\mathbb{R}_{\geq 0}$, $\mathbb{R}_{> 0}$ be the set of reals, non-negative reals, and positive reals respectively. Let $\mathcal{S}^{n}_{++}$ denote set of symmetric positive-definite matrices in $\mathbb{R}^{n \times n}$. Let $\mathcal{N}(\mu, \Sigma)$ denote a normal distribution with mean $\mu$ and covariance $\Sigma \in \mathcal{S}^{n}_{++}$. The $Q \in \pd^{n}$, norm of a vector $x \in \R^n$ is denoted $\norm{x}_Q = \sqrt{x^T Q x}.$

\subsection{System Description}
Consider a multi-agent system, in which each robotic system $i \in \Rcal = \{0, 1, \cdots, N-1 \}$, referred to as \textbf{\textit{rechargeable robot}}, comprises the robot and battery discharge dynamics: 

\begin{equation}
\label{eqn:rechargeable}
     \Dot{\chi}^i =
    \begin{bmatrix}
        \Dot{x}^i \\
        \Dot{e}^i 
    \end{bmatrix}
    =
    f^i(\chi^i, u^i) \\
    =
    \begin{bmatrix}
        f_r^i(x^i, u^i) \\
        f_e^i(e^i)
    \end{bmatrix},
\end{equation}
where $N = |\Rcal|$ is the cardinality of the set $\Rcal$, $\chi^i = \begin{bmatrix}{x^i}^T, & e^i\end{bmatrix}^T \in \Zcal^i_r \subset \R^{n+1}$ is the $i^{th}$ robotic system state consisting of the robot state $x^i \in \mathcal{X}^i_r \subset \mathbb{R}^{n}$ and its State-of-Charge (SoC) $e^i \in \mathbb{R}_{\geq 0}$. $u^i \in \Ucal^i_r \subset \mathbb{R}^{m}$ is the control input, $f^i: \Zcal^i_r \times \Ucal^i_r \rightarrow \mathbb{R}^{n+1}$ defines the continuous-time robotic system dynamics, $f_r^i: \Xcal^i_r \times \Ucal^i_r \rightarrow \mathbb{R}^n$ define robot dynamics and $f_e^i: \mathbb{R}_{\geq 0} \to  \mathbb{R}$ define worst-case battery discharge dynamics. We also consider the continuous-time dynamics of the mobile charging station (referred to as \textbf{\textit{mobile charging robot}}):
\begin{subequations}
\label{eqn:mobile_charging}
\eqn{
    \Dot{x}^c &= f_c(x^c, u^c) + w(t),   &&w(t) \sim \mathcal{N} (0, W(t)), \label{eqn: charge_dynamics} \\
    y^c &= z(x^c) + v(t), &&v(t) \sim \mathcal{N} (0, V(t)), \label{eqn:obs_nl_model}
}    
\end{subequations}
where $x^c \in \mathcal{X}_c \subset \mathbb{R}^{c}$ is the charging station state, $u^c \in \mathcal{U}_c \subset \mathbb{R}^{s}$ is the charging station control input, $f_c: \Xcal_c \times \Ucal_c \rightarrow \mathbb{R}^{c}$ defines the continuous-time system dynamics for the mobile charging, $w(t)$ is the time-varying process noise with zero mean and known variance $W(t) \in \mathbb{R}_{\geq 0} $, $y \in \mathbb{R}^{c}$ is the measurement, $z : \mathbb{R}^{c} \to \mathbb{R}^{c}$ is the observation model, and $v(t)$ is the time-varying measurement noise with zero mean, and known covariance $V(t)$.  
\subsection{Problem Statement}
\label{sec:problem_statemebt}
Consider a team of $N + 1$ robots performing a persistent mission. Among them, $N$ robots are the rechargeable robots, require recharging, while $1$ robot serves as the mobile charging robot that doesn't require recharging.\footnote{This could represent a ground vehicle with a battery lasting a few hours. This assumption has been also made in previous works \cite{cooperative_1, MOBILE_RAL_2022}.} The rechargeable robots model the mobile charging robot using \eqref{eqn:mobile_charging}.

We assume the existence of a high-level trajectory planner that designs the nominal trajectories for all $N+1$ robots based on the high-level mission objectives. The objectives for the rechargeable robots are as follows:

1) Track the nominal trajectories and deviate only when returning for recharging.

2) Ensure mutually exclusive use of the mobile charging robot, which strictly follows its nominal trajectory.

Now we formulate the problem mathematically. We define $\Tcal^i$ as the set of times $i^{th}$ robot returns to the charging station:
\eqn{
\Tcal^i = \{t_{0}^i, t_{1}^i, \cdots, t_{m}^i, \cdots\},  \forall i \in \Rcal, \forall m \in \nonnegintegers
}
where $t_{m}^i$ represents the $m^{th}$ return time of the $i^{th}$ rechargeable robot. Let $\Tcal = \cup_{i \in \Rcal} \Tcal^i$ be the union of return times for all robots. We now define two conditions that must hold for all times $t \in [t_0, \infty)$ to achieve the objectives stated above:
\begin{subequations}
\eqn{
&e^i(t) \geq {e^i_{min}} \quad  \quad \forall t \in [t_0, \infty),  \forall i \in \Rcal  \label{eq:min_energy_constraints}\\
&|t^{i_1}_{m_1} - t^{i_2}_{m_2}| > T_{\delta}  \quad  \forall t^{i_1}_{m_1},t^{i_2}_{m_2}  \in \Tcal \label{eq:gap_lower_bound}
}
\end{subequations}
Condition \eqref{eq:min_energy_constraints}, the \textbf{\textit{minimum SoC condition}}, defines the required minimum battery SoC for all rechargeable robots. Condition \eqref{eq:gap_lower_bound}, the \textbf{\textit{minimum gap condition}}, ensures a sufficient time gap between the returns of two robots to avoid charging conflicts. The term $T_{\delta} = T_{ch} + T_{bf}$ represents the charging duration and the buffer time needed for a robot to resume its mission before the next robot arrives. In summary, the problem statement is:
\begin{prob}
    \label{prob1}
    Given the nominal high level planner for all $N+1$ robots, design an algorithm to track the nominal trajectories for the $N$ rechargeable robots while ensuring that conditions \eqref{eq:min_energy_constraints}, \eqref{eq:gap_lower_bound} are met at all times. 
\end{prob}
Our goal in solving \cref{prob1} is to develop an online algorithm for nonlinear robot models that does not assume a precomputed nominal trajectory for the entire mission but instead only requires a short horizon nominal trajectory at each decision iteration, allowing the high-level planner to adapt dynamically.  The algorithm must handle varying discharge rates, account for erroneous state information about the mobile charging station available to rechargeable robots, and remain computationally efficient for scalability.

\section{Method Overview, Motivation \& Key Ideas}
We propose \mesch{} as a solution to \cref{prob1} within the persistent planning framework (\cref{fig:mesch_overview}). As a low-level module, \mesch{} ensures task persistence. The solution follows four steps, with the \mesch{} module running every $T_E$ seconds at discrete time steps $t_j = jT_E$, where $j \in \mathbb{Z}_0$:

\begin{itemize}
\item Compute the rendezvous point where the rechargeable robot will return for recharging.
\item Determine the reserve energy at the rendezvous to account for uncertainty in the charging station's position.
\item Construct a trajectory that follows the portion of the nominal trajectory and then reaches the rendezvous point. We refer to this as a \textit{candidate trajectory.}
\item Commit the new candidate trajectory if it satisfies the minimum energy condition \eqref{eq:min_energy_constraints} and the minimum gap condition \eqref{eq:gap_lower_bound}. The result is the \textit{committed trajectory}.
\end{itemize}
Before detailing each step, we first explain how \mesch{} evaluates the satisfaction of conditions \eqref{eq:min_energy_constraints} and \eqref{eq:gap_lower_bound}. This is one of our key contributions, and we explain it by first discussing its motivation and then describing its mechanism.

\subsection{Method motivation}

Consider $N$ quadrotors sharing a mobile charging rover, as shown in \cref{fig:traj_eware}. To prevent charging conflicts, we propose a scheduling method based on two principles.

First, if multiple robots are predicted to arrive simultaneously, one is rescheduled to arrive earlier. Second, if robots naturally visit the charging station at different times due to varying discharge profiles, the algorithm checks that each robot has enough energy to continue its mission, ensuring that the minimum energy condition is never violated.

To implement this approach, we introduce two modules: \gware{} and \eware{}. The \gware{} module maintains the minimum time gap between charging sessions by constructing \textbf{\textit{gap flags}} and rescheduling robots to arrive early when conflicts arise, ensuring that the minimum gap condition \eqref{eq:gap_lower_bound} is always satisfied. Once all gap constraints are met, the module \eware{} checks whether each robot has enough energy to continue its mission, ensuring that the minimum energy condition \eqref{eq:min_energy_constraints} is always met.

\subsection{Construction of Gap flags}

We begin by describing the construction of gap flags and their role in preventing charging conflicts. At each iteration of \mesch{}, rechargeable robots are sorted by their remaining flight time into the ordered set $\Rcal' = \{0', 1', \dots, N'-1\}$, where $0'$ has the least flight time. For each robot $k \in \Rcal' \setminus \{0'\}$, a gap flag is constructed relative to $0'$ as:
\eqn{
G^{k} = T_{F,j}^{k} > (T_R + T_E + kT_{\delta}),
}
where $T_{F,j}^{k}$ is $k^{th}$ robot remaining flight time at time $t_j$, $T_R$ is the time to reach the charging station, $T_E$ is the decision interval, and $T_{\delta}$ includes the charging duration and the buffer time required to resume the mission.
\begin{table}[h!]
\scriptsize
\setlength\extrarowheight{-3pt}
\begin{tabular}{p{0.10\linewidth} |
p{0.75\linewidth}}
\midrule
\textbf{Symbol} & \quad \quad \quad \quad \quad \quad \quad \textbf{Definition}  \\ 
\midrule
\multicolumn{2}{l}{Indices} \\
\midrule
$i$ & Rechargeable robot index \\
$j$ & \mesch{} iteration index  \\
$k$ & Rechargeable robot index in sorted list \\
\midrule
\multicolumn{2}{l}{Constant shared time horizons} \\
\midrule
$T_{\delta}$ & $T_{\delta} = T_{ch} + T_{bf}$ Charging + Buffer time  \\
$T_N$ & Nominal trajectory horizon of the rechargeable robot available at time $t_j$  \\
$T_R$ & Charging robot nominal trajectory horizon available at $t_j$ / Time taken by the rechargeable robot to reach the charging station \\
$T_E$ & Time interval between $j$ and $j + 1$ iteration  \\


\midrule
\multicolumn{2}{l}{Dynamic time horizons for robot $i$ computed at $t_j$} \\
\midrule
$T^i_{L,j}$ & Worst-case landing time  \\
$T^i_{C,j}$ & Candidate trajectory ($T^i_{C,j} = T_R - T^i_{L,j}$) \\
$T^i_{B,j}$ & back-to-base trajectory ($T^i_{B,j} = T^i_{C,j} - T_N$)  \\
$T^k_{F,j}$ &  Remaining battery time of the $k^{th}$ robot in the sorted list at time $t_j$ \\
\midrule
\multicolumn{2}{l}{Time points} \\
\midrule
$t_j$ & Start time of iteration $j$ \\
$t_{j,N}$ & $t_j + T_N$ \\
$t^i_{j,C}$ & $t_j + T^i_{C,j}$ \\
$t_{j,R}$ & $t_j + T_R$ \\
$t_{m}^i$ &  $m^{th}$ time $i^{th}$ robot returns for recharging   \\
\midrule
\end{tabular}
\vspace{-8pt}
\caption{Time and Index Notation at a glance}
\vspace{-23pt}
\label{table:1}
\end{table}
These flags enforce a minimum gap of $kT_{\delta}$ between robot $0'$ and robot $k$ in $\Rcal'$. For example, the minimum gap between the first and third robots is $2T_{\delta}$. If any gap flag is not satisfied, the robot with the least remaining flight time, i.e., $0'$, is rescheduled for recharging. The satisfaction of the gap flag condition guarantees that there will be at least $T_{\delta}$ between successive charging sessions.

\section{Multi-Agent Energy-Aware Scheduling}
\label{meSch}
In this section, we present the full proposed solution to \cref{prob1}, using $N$ rechargeable quadrotors and one mobile charging rover, without loss of generality. After establishing the construction of gap flags, we demonstrate how they are iteratively checked within the full solution scheme to ensure conditions \eqref{eq:min_energy_constraints} and \eqref{eq:gap_lower_bound} hold for all $t \in [0, \infty)$. We also discuss how the proposed method accounts for the uncertainty in the position of the mobile charging robot. This solution is developed under a few key assumptions: 

\begin{assumption}
\label{assumption:1}
At each iteration of \mesch{}, the nominal trajectories of the rechargeable robots are known for $T_N$ seconds, and the mobile charging robot's for $T_R$ seconds.
\end{assumption}
We also define the notation for trajectories. Let $x^i([t_j, t_j + T_N]; t_j, x^i_j)$ represent the nominal trajectory for the $i^{th}$ rechargeable robot at time $t_j$, starting from state $x^i_j$ and defined over a time horizon of $T_N$ seconds. We denote this as $x^{i, nom}_j$. The same notation applies to other trajectories. An overview of the notation is provided in \cref{table:1}.

\subsection{Estimating Rendezvous Point}
\label{renp}
At the $j^{th}$ iteration of \mesch{}, we estimate the mobile charging robot's position at $t_j + T_R$, i.e., $\hat{x}^c(t_{j, R})$, and place the rendezvous point $d$ meters above it. The rechargeable robots will return to this point, as shown in \cref{fig:traj_eware}.

Given the current state estimate $\hat{x}^c(t_j)$ and its covariance $\Sigma^c(t_j)$ from the EKF, we use the EKF predict equations \cite{applied_optimal_estimation} to compute the mobile charging robot state estimate at $t_{j,R}$, i.e. $\hat{x}^c(t_{j,R})$ and $\Sigma^c(t_{j,R})$. The rendezvous point $x^{rp}_j \in \R^n$ is then computed as follows:
\begin{figure} [t]
  \centering
  \includegraphics[width=1.0\columnwidth]{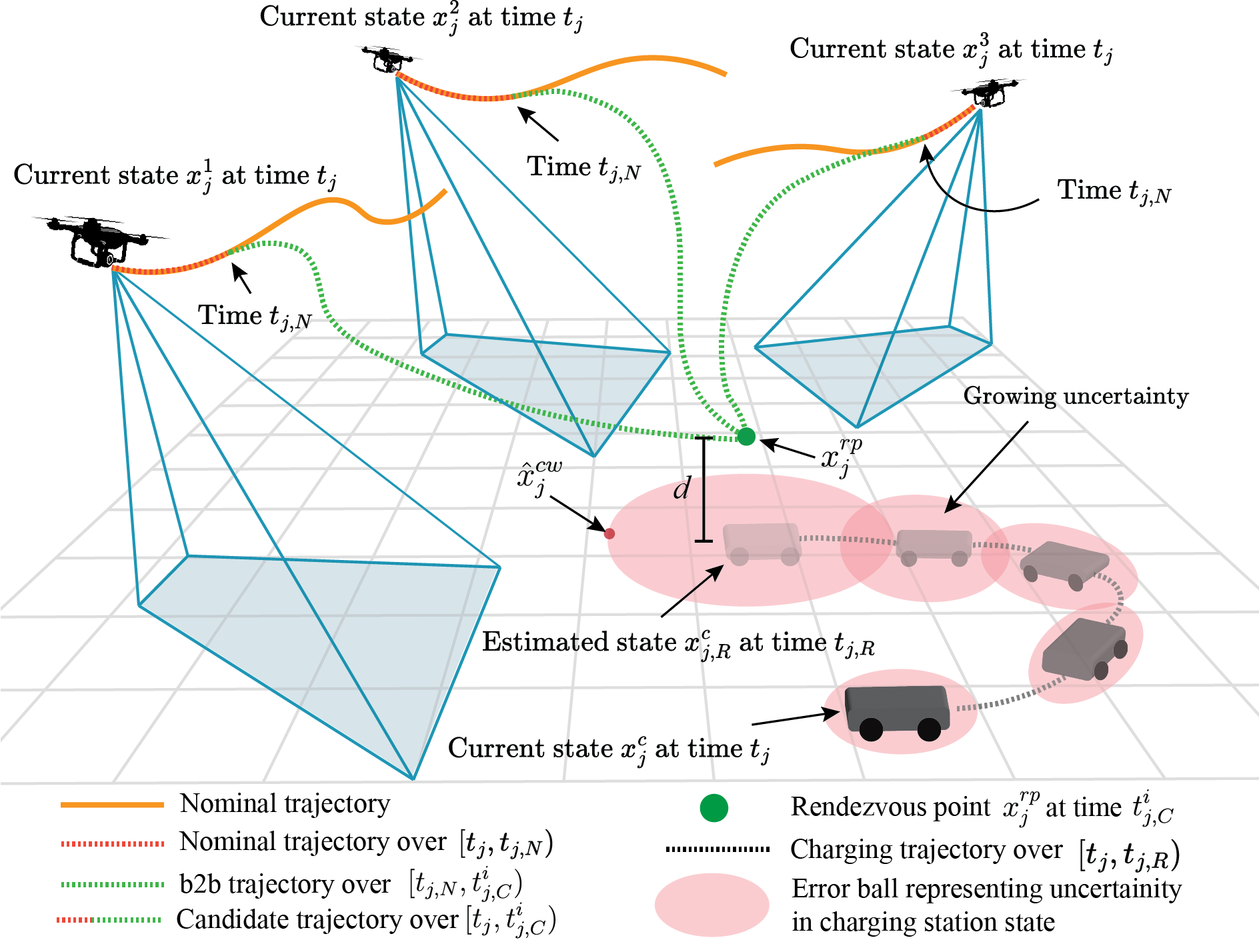}
  \caption{This figure illustrates the generation of candidate trajectories at time $t_j$. All the candidate trajectories terminate at the rendezvous point $x^{rp}_j$ at time $t^i_{j,C}$.}
  \vspace{-20pt}
  \label{fig:traj_eware}
\end{figure}
\eqn{
x^{rp}_j = 
    \begin{bmatrix}
        \Psi(\hat{x}^c(t_{j,R})) \\
        \textbf{0}_{n-2} 
    \end{bmatrix}
    +
    \begin{bmatrix}
        \textbf{0}_{2} \\
        d \\
        \textbf{0}_{n-3}
    \end{bmatrix}    
}
where  $\Psi : \mathbb{R}^{c} \to \mathbb{R}^{2}$ is a mapping that
returns the 2-D position coordinates, $d \in \mathbb{R}_{> 0}$ is added to the z-dim of the state, and $n$ is the rechargeable robot state dimension \eqref{eqn:rechargeable}. The rendezvous point corresponds to the hover reference state $x_{j}^{rp}$ for the rechargeable robot, positioned $d$ meters above the predicted position of the mobile charging robot.


\subsection{Reserve Energy for Uncertainty-Aware Landing}

Along with the rendezvous point $x^{rp}_j$, we also compute the remaining energy the robot must have at the rendezvous point to account for uncertainty in the mobile charging robot's position for landing. This corresponds to the energy cost of going from rendezvous point $x^{rp}_j$ to the furthest state $\hat{x}^{cw}_j$ within the 95\% confidence interval covariance ellipse.
Now, we compute the furthest point on the boundary of the 95\% confidence ellipse as follows:
\eqn{
\hat{x}^{cw}_j = \hat{x}^c(t_{j,R}) + q_{max} \sqrt{\chi^2_{c, 0.95} \lambda_{max}} 
}
where $\lambda_{max} \in \R$ is the largest eigenvalue of the covariance matrix $\Sigma^c(t_{j,R})$, $q_{max} \in \R^c$ is the eigenvector corresponding to $\lambda_{max}$, and $\chi^2_{c, 0.95}$ corresponds to the value from the chi-squared distribution with $c$ degrees of freedom in the 95\% confidence interval. To compute the reserve energy, we formulate the following problem $\forall i \in \Rcal$:
\begin{subequations}
\label{eq:landing_overall}
\eqn{
        \min_{\chi^i(t),u^i(t), t^i_{f}} \ &  t^i_{f} \label{eq:landing_prob}\\
        \textrm{s.t. } \ & \chi^i(t^i_{0}) = \chi^i_{rp} \\
        &  \Dot{\chi} = f_r^i(\chi^i, u^i) \\
        & x^i(t^i_{f}) = \hat{x}^{cw}_j
}
\end{subequations}
where $\chi^i_{rp} = [[x^{rp}_j]^T, e^i_0]^T $ is the initial system state comprising of $x^{rp}_j \in \R^n$ and the energy $e_0 \in \mathbb{R}_{> 0}$. The reserve energy $e^{i, res}_{j}$ and landing time $T^i_{L,j}$ are computed as follows:
\begin{subequations}
    \eqn{e_{j}^{i, res} &= e^i(t^i_{f}) - e^i(t^i_{0}) \quad \forall i \in \Rcal \\
    T^i_{L,j} &= t^i_{f} - t^i_{0} \quad \forall i \in \Rcal
    }
\end{subequations}
\begin{algorithm}
  \footnotesize
  \label{alg:full}
  \DontPrintSemicolon
  \caption{The \mesch{} Algorithm}\label{alg:mesch}
  \SetKw{KwParams}{Parameters:}
  \SetKwProg{When}{When}{}{end}
  \SetKwProg{Fn}{function}{}{end}  
  \Fn{$\mesch$($ x_j^{i, can}  , x_{j-1}^{i, com}, e^{i, res}_j$)}{
    GapViolation = \gware($x_j^{i, can}  , x_{j-1}^{i, com}$)\;  
    \If{GapViolation == 1}{
        \Return
    }
    \eware($x_j^{i, can}, x_{j-1}^{i, com}, e^{i, res}_j$)
}
\end{algorithm}
\subsection{Construction of Candidate Trajectories} 

Now, we generate the candidate trajectories for all rechargeable robots to reach the rendezvous point $x^{rp}_j$ from the current state $x^i(t_j)$ within $T^i_{C,j} = T_R - T^i_{L,j}$ s.

Given nominal trajectories $x^{i, nom}_j$ $\forall i \in \Rcal$, we construct a candidate trajectory that tracks a portion of the nominal trajectory for $T_N$ s and then reaches the rendezvous point $x^{rp}_j$ within $T^i_{B,j}$ = $T^i_{C,j} - T_N$ s. For the $i^{th}$ rechargeable robot, the candidate trajectory is constructed by concatenating the nominal trajectory with a \textbf{\textit{back-to-base (b2b) trajectory}}. 
Let the $i^{th}$ rechargeable robot state at time $t_j$ be $x^i_j \in \Xcal$ and the system state at $t_j$ be $\Xcal^i_j \in \Zcal^i_r$. We construct a b2b trajectory $x^{i,b2b}_j$ defined over interval [$t_{j,N}, t^i_{j,C}$] by solving:
\begin{subequations}
\label{eq:b2b_prob_overall}
\eqn{
        \min_{x^i(t),u^i(t)} \ &  \int_{t_{j,N}}^{t^i_{j,C}} \norm{x^i(t) - x^{rp}_j}_{\mathbf{Q}}^2 + \norm{u^i(t)}_{\mathbf{R}}^2 dt \label{eq:b2b_prob}\\
        \textrm{s.t. } \ & x^i(t_{j,N}) = x_{j}^{i, nom}(t_{j,N}) \\
        &  \Dot{x}^i = f_r^i(x^i, u^i) \\
        & x^i(t^i_{j,C}) = x^{rp}_j
}
\end{subequations}
where $\mathbf{Q} \in \pd^{n}$ and $\mathbf{R} \in \pd^m$ weights state cost and  control cost respectively. 
Once b2b trajectory $x_{j}^{i,b2b}$ is generated, we numerically construct the system candidate trajectory
\eqn{
    \chi^{i,can}_j &= \begin{cases} x^{i,can}_j(t), & t \in [t_j, t^i_{j,C}) \\
        e^{i,can}_j(t), & t \in [t_j, t^i_{j,C})
   \end{cases} 
}over a time interval [$t_j, t^i_{j,C}$) by solving the initial value problem for each rechargeable robot system, i.e.
\begin{subequations}
\eqn{
\label{eqn:candidate}
    \Dot{\chi}^i &= f(\chi^i, u^i(t)),\\
    \chi^i(t_j) &= \chi^i_j \\
    u^i(t) &= \begin{cases} \pi^i_r(\chi^i,  x^{i,nom}_j(t)), & t \in [t_j, t_{j,N}) \\
        \pi^i_r (\chi,  x_{j}^{i,b2b}(t)), &  t \in [t_{j,N}, t^i_{j,C})
   \end{cases} 
}
\end{subequations}
\begin{algorithm}
  \footnotesize
  \label{alg:gware}
  \DontPrintSemicolon
  \caption{The \gware{} algorithm}\label{alg:gware}
  \SetKw{KwParams}{Parameters:}
  \SetKwProg{When}{When}{}{end}
  \SetKwProg{Fn}{function}{}{end}  
  \Fn{\gware($ x_j^{i, can}  , x_{j-1}^{i, com}$)}{
    \tcp{Sort $x_j^{i, can}$ based on $T_{F}^i$} 
  \For{ $k \in \textbf{sorted list}$ $ \Rcal' = \{1', 2', \cdots, N'-1\}$:}{
    $G^{k} = (T_{F}^{k} - T_R - T_{E}) > k(T_{\delta} + T_\delta)$ \;
    \If{ ($G^{k} ==$ 0) $\wedge$ (k.status $\neq$ charging) }{
        $x_j^{0', com} \gets x_{j-1}^{0', com}$ \;
        $x_j^{k, com} \gets x_{j}^{k, can} \quad \forall k \in \Rcal' $ \;
        Initiate landing for the $0^{'th}$ robot at $t^i_{j,C}$ \;
    \Return True
    }
    }
}
\end{algorithm}
\begin{algorithm}
  \label{alg:eware}
  \footnotesize
  \DontPrintSemicolon
  \caption{The \eware{} algorithm}\label{alg:eware}
  \SetKw{KwParams}{Parameters:}
  \SetKwProg{When}{When}{}{end}
  \SetKwProg{Fn}{function}{}{end}  
  \Fn{\eware($ x_j^{i, can}  , x_{j-1}^{i, com}, e^{i, res}_j$)}{
  \For{ $i \in \{0, 1, \cdots, N-1\}$:}{
    \If{ $e^i(t) \geq e^{i, res}_j \ \forall t \in [t_j, t^i_{j,C}]$ }{
        $x_j^{i, com} \gets x_j^{i, can}$
    }
    \Else {
        $x_j^{i, com} \gets x_{j-1}^{i, com}$ \;
        Initiate landing for the $i^{th}$ robot at $t^i_{j,C}$  \;  
    }
    }
}
\end{algorithm}
where $\pi^i_r: \mathcal{Z}_r^i \times \mathcal{X}_r^i \rightarrow \mathcal{U}_r^i$ is a control policy to track the portion of the nominal trajectory and the b2b trajectory. \cref{fig:traj_eware} shows the candidate trajectory generation process with 3 rechargeable robots and 1 mobile charging robot.

\subsection{Energy-aware Scheduling}

Given the candidate trajectory and the reserve energy for each rechargeable robot $i \in \Rcal$, we check if the minimum SoC condition \eqref{eq:min_energy_constraints} and the minimum gap condition \eqref{eq:gap_lower_bound} are satisfied throughout the candidate trajectory. The overall algorithm described in \cref{alg:mesch} consists of the two subroutines: \gware{} (gap-aware) and \eware{} (energy-aware). 
\begin{figure*}[t]
  \centering
  \includegraphics[width=2.0\columnwidth]{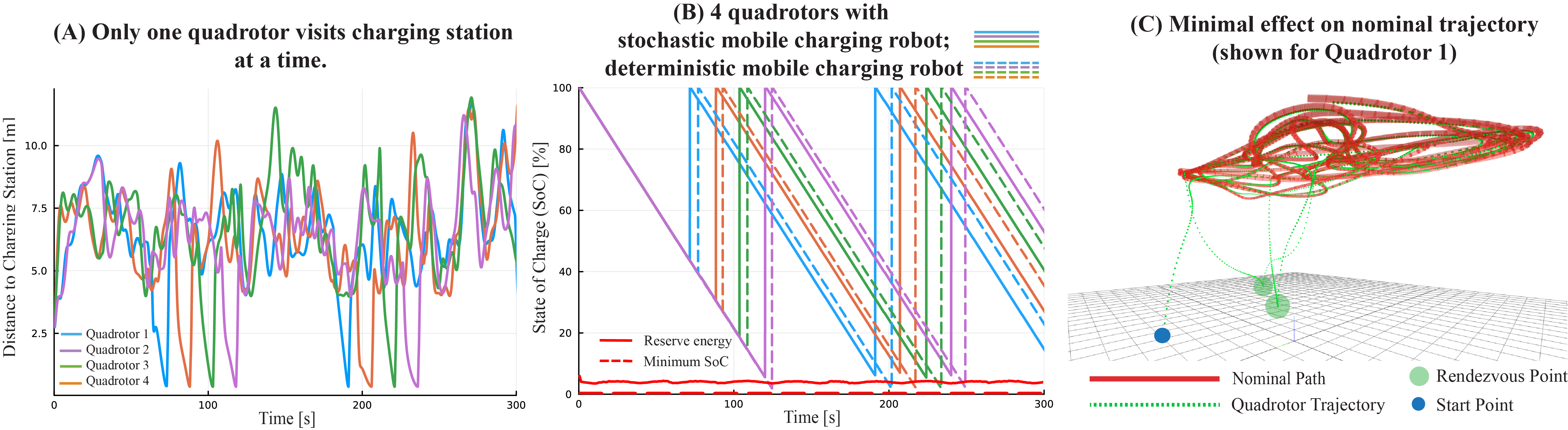}
  \caption{Results from the Multi-Agent Energy-Aware Persistent Ergodic Search Case study}
  \vspace{-17pt}
  \label{fig:ergodic_study}
\end{figure*}
\subsubsection{Gap-aware (\gware)}
\gware{} described in \cref{alg:gware} checks if the rechargeable robots would continue to have the gap of $T_{\delta}$ seconds between their expected returns if the candidate trajectories were committed.  

(Lines 2-4) Here the gap flags are constructed for each $k^{th}$ robot that is not currently charging or returning, relative to the first robot in the sorted list (the $0^{\text{th}}$ robot)
\eqn{
\label{gap_flag}
G^{k} = T_{F,j}^{k} > (T_R + T_{E} + kT_{\delta})
}
Satisfaction of the gap flag condition at the $j^{th}$ iteration implies that rechargeable robots are estimated to have at least $T_{\delta}$ of the gap between their expected returns over the time interval $[t_j, t_{j,R})$, i.e. $\forall t^{i_1}_{m_1},t^{i_2}_{m_2}  \in \Tcal $:  
    \eqn{
&T_{F,j}^{k} > (T_R + T_{E} + kT_{\delta}) \\
&\implies |t^{i_1}_{m_1} - t^{i_2}_{m_2}| > T_{\delta}   &\forall t \in [t_{j}, t_{j, R} )
}
(Lines 5-7) If any gap flags are false, the committed trajectory of the first rechargeable robot in the sorted list remains unchanged, and it returns to the charging station. Meanwhile, the candidate trajectories are committed for the subsequent rechargeable robots in the sorted list. 


\subsubsection{Energy-aware (\eware)} If no gap flag violations occur, indicating that all rechargeable robots have sufficient gaps between their expected return for recharging, we proceed to check if each robot has adequate energy to continue the mission using \eware{} described in \cref{alg:eware}. 

(Lines 3-6) We assess whether each rechargeable robot can reach the charging station without depleting its energy below the reserve level while following the candidate trajectory. We refer to this condition as the \textbf{\textit{Reserve SoC Condition}}:
\eqn{
\label{min_soc}
e^i(t) > e^{i, res}_j \ \forall t \in [t_j, t^i_{j,C}]
}
If successful, the candidate trajectory replaces the current committed one. For the returning robot, a landing controller is assumed to exist:
\begin{assumption} When the returning rechargeable robot reaches rendezvous point $x^c_{rp}$ at $t^i_{j,C}$, there exists a landing controller $\pi^i_l : [t^i_{j,C}, t_{j,R}) \times \Xcal_r^i \rightarrow \Ucal$ that guides the rechargeable robot to the mobile charging robot.
\end{assumption}


\begin{thm}
    \label{thm1}
    Suppose that at $j =0$ the Gap flag condition \eqref{gap_flag} and the Reserve SoC condition \eqref{min_soc} are satisfied. Therefore, the conditions  \eqref{eq:min_energy_constraints} and \eqref{eq:gap_lower_bound} are satisfied for $[t_{0}, t^i_{0, R} )$. Following this, if  solution exist for \eqref{eq:landing_overall} and \eqref{eq:b2b_prob_overall} and the committed trajectories $x_j^{i, com} = x^i([t_j, t^i_{j,C}]; t_j, x^i_j)$ are determined, $\forall i \in \Rcal$, using the \cref{alg:mesch}, then \eqref{eq:min_energy_constraints} and \eqref{eq:gap_lower_bound} are satisfied $\forall t \in [t_j, t_{j-1,R})$, $\forall j \in \posintegers$. 
\end{thm}

\begin{proof} Please see \cref{proof_theorem_1} for detailed proof. 
\end{proof}
Given that the conditions \eqref{eq:min_energy_constraints} and \eqref{eq:gap_lower_bound} are met for the interval $[t_j, t_{j-1,R})$ $\forall j \in \posintegers$ , and $t_j < t_{j-1, N} < t_{j-1, R}$, the conditions \eqref{eq:min_energy_constraints} and \eqref{eq:gap_lower_bound} are satisfied over $[t_0, \infty)$. 
\begin{table*}[h]
    \scriptsize
    \centering
    \caption{Comparison of baseline methods and proposed \mesch{}}
    \renewcommand{\arraystretch}{1.3} 
    \begin{tabular}{>{\centering\arraybackslash}p{2cm}>{\centering\arraybackslash}p{2.6cm}>{\centering\arraybackslash}p{2.0cm}>{\centering\arraybackslash}p{1.3cm}>{\centering\arraybackslash}p{1.3cm}>{\centering\arraybackslash}p{1.3cm}>{\centering\arraybackslash}p{1.4cm}>{\centering\arraybackslash}p{1.8cm}}
        \hline
        \textbf{Method} & \textbf{Robot Model (Supports Nonlinear Dynamics)} & \textbf{Requires Different Deployment Times} & \textbf{Total Recharging Visits} & \textbf{Concurrent Charging Visits} & \textbf{Minimum Energy Violations} & \textbf{Scalability Analysis} & \textbf{Supports Mobile Charging Station}\\ 
        \hline
        Baseline 1 \cite{persis_Fouad} & 3D Single integrator (\textcolor{red}{No}) & No & 8 & 0 & 0 & Not provided & \textcolor{red}{No}\\ 
        \hline
        Baseline 2 \cite{bentz2018complete}& 3D Quadrotor \cite{beard2008quadrotor} (Yes) & \textcolor{red}{Yes} & 8 & 0 & 0 & Not provided & \textcolor{red}{No}\\ 
        \hline
        Baseline 3  \cite{naveed2023eclares}& 3D Quadrotor \cite{jackson2021planning} (Yes) & No & 8 & \textcolor{red}{2} & 0 & Not provided & \textcolor{red}{No}\\ 
        \hline
        Baseline 4 [\mesch{} with only \gware] & 3D Quadrotor \cite{jackson2021planning} (Yes) & No & 4 & 0 & \textcolor{red}{4} & $\mathcal{O}(N \log N)$ & Yes\\ 
        \hline
        Proposed [\mesch{}] & 3D Quadrotor \cite{jackson2021planning} (Yes) & No & 8 & 0 & 0 & $\mathcal{O}(N \log N)$ & Yes\\ 
        \hline
    \end{tabular}
    \label{tab:meSch_comparison}
\end{table*}
\begin{figure} [t]
    
  \centering
  \includegraphics[width=1.0\columnwidth]{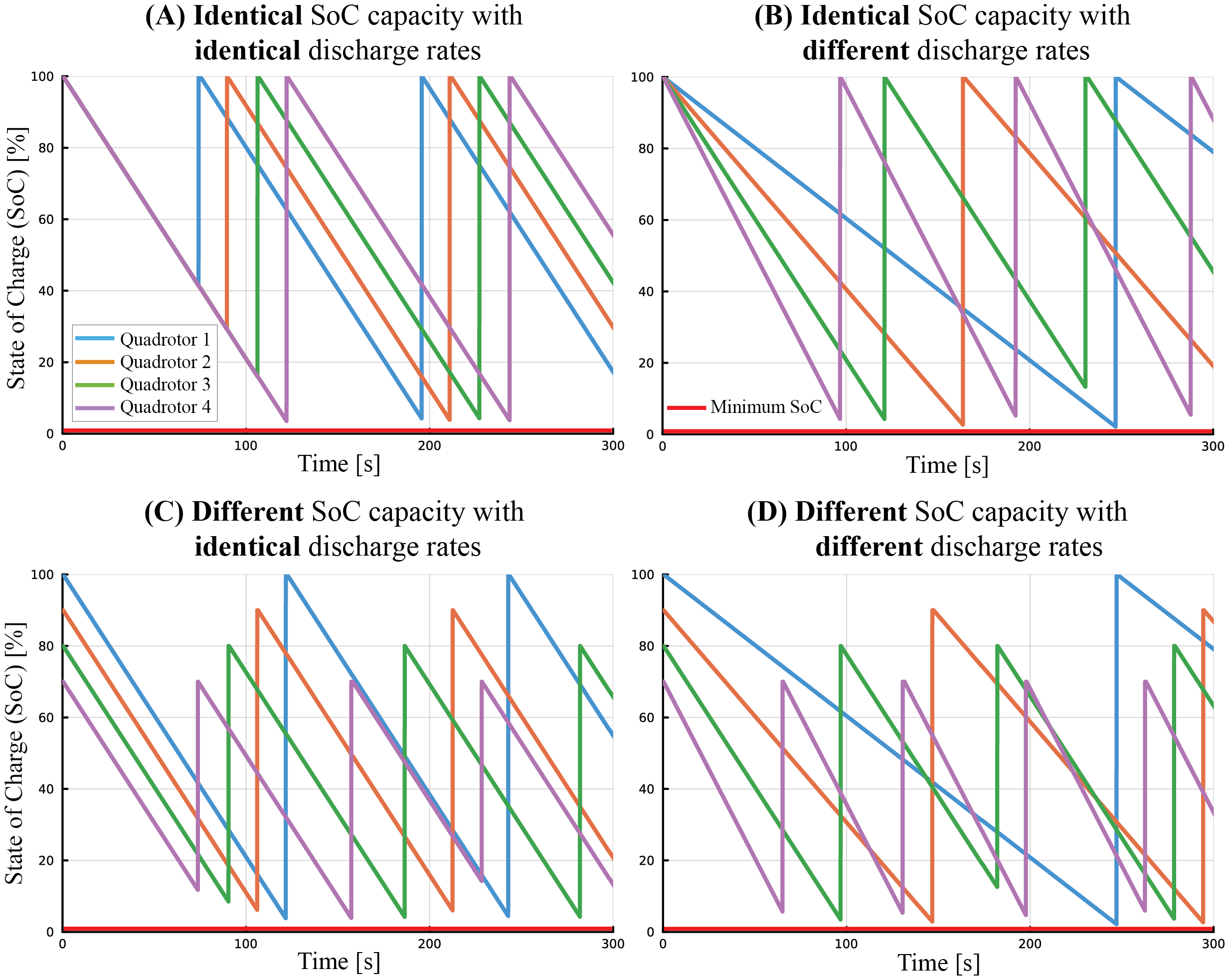}
  \caption{These plots show results for the scenarios when four quadrotors have different SoC capacities and different discharge rates. The plots validate that quadrotors always maintain the minimum of $(T_c + T_\delta)$ gap while visiting the charging station.}
  \vspace{-23pt}
  \label{fig:soc_time}
\end{figure}
\section{Results \& Discussion}
\label{sec:results}

In this section, we evaluate \mesch{} through case studies and hardware experiments. We use quadrotors with 3D nonlinear dynamics from \cite[Eq. (10)]{jackson2021planning} as rechargeable robots and rovers with unicycle models as mobile charging robots. We assume instantaneous recharging ($T_{ch} = 0.0$ s) and a buffer time of $T_{bf} = 15.0$ s, with $T_N = 2.0$ s and $T_R = 18.0$ s, consistent across all experiments.

To generate b2b trajectories, we solve \eqref{eq:b2b_prob_overall} using MPC with the reduced linear quadrotor dynamics from \cite{jackson2021planning}. We use an LQR controller for \eqref{eq:landing_overall} and an LQG controller for landing. Trajectories are generated at 1.0 Hz and tracked at 50.0 Hz with zero-order hold, using the RK4 integration.

\subsubsection{\textbf{Multi-Agent Energy-Aware Persistent Ergodic Search}}

We evaluate \mesch{} by simulating a scenario with four quadrotors and one rover exploring a $10 \times 10$ m domain. Nominal trajectories are collision-free ergodic trajectories with a horizon of $T_H = 30.0$ s. All quadrotors use discharge dynamics $\Dot{e} = -0.667$. Coverage objectives are omitted for brevity. \Cref{fig:ergodic_study} summarizes the results.

\Cref{fig:ergodic_study} (A) shows that only one quadrotor is at the charging station at any time. Collision avoidance is ensured by tuning the $T_\delta$ parameter. \Cref{fig:ergodic_study} (B) plots the SoC evolution for two cases: with a stochastic mobile charging robot and with a deterministic one. In the stochastic case, the quadrotors' SoC never drops below the reserve energy level and in the deterministic case, the SoC levels never drop below zero. Finally, \cref{fig:ergodic_study} (C) shows that the nominal ergodic trajectory (in red) is only affected when quadrotor 1 returns for recharging.
\begin{figure} [t]
  \centering
  \includegraphics[width=0.8\columnwidth]{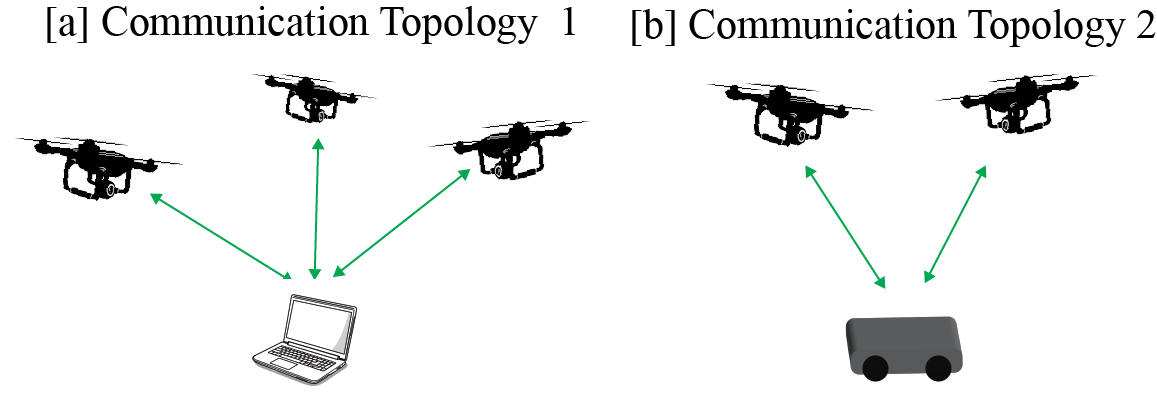}
  \caption{The communication architecture of the system.}
  \vspace{-23pt}
  \label{fig:comms}
\end{figure}

\subsubsection{\textbf{Comparison to the Baseline Methods}} We compare \mesch{} to baseline methods using seven metrics, as shown in \cref{tab:meSch_comparison}. For each method, four robots are used with the same discharge model, $\Dot{e} = -0.667$. The total recharging visits are the same across all methods, except for Baseline 4, which focuses only on the timing of robot visits and does not account for the minimum energy requirements.

Compared to Baseline 1, \mesch{} supports nonlinear dynamic models, making it more applicable to real-world robotic platforms, as demonstrated with 3D quadrotor dynamics \cite{jackson2021planning}. Unlike Baseline 2, \mesch{} effectively handles both identical and varying discharge rates and state-of-charge (SoC) capacities without requiring robots to be deployed at different times. Deploying robots at different times reduces the number of robots available for the mission at any given moment, limiting overall efficiency. By allowing all robots to be deployed simultaneously, \mesch{} simplifies mission planning and increases adaptability to different discharge patterns, as shown in \cref{fig:soc_time} with four quadrotors. Compared to Baseline 3, \mesch{} eliminates simultaneous charging station visits. In Baseline 3, four robots returned concurrently on two occasions, leading to a violation of \eqref{eq:gap_lower_bound}. While Baseline 4, which only includes the \gware{} module from \mesch{}, successfully avoids overlapping visits, it fails to enforce minimum energy constraints, resulting in a violation of \eqref{eq:min_energy_constraints}. Finally, none of the baseline methods support mobile charging stations, a limitation in the environments where fixed charging locations may be infeasible. By addressing this gap, \mesch{} enhances mission endurance and scalability.



\begin{figure} [t]
  \centering
  \includegraphics[width=1.0\columnwidth]{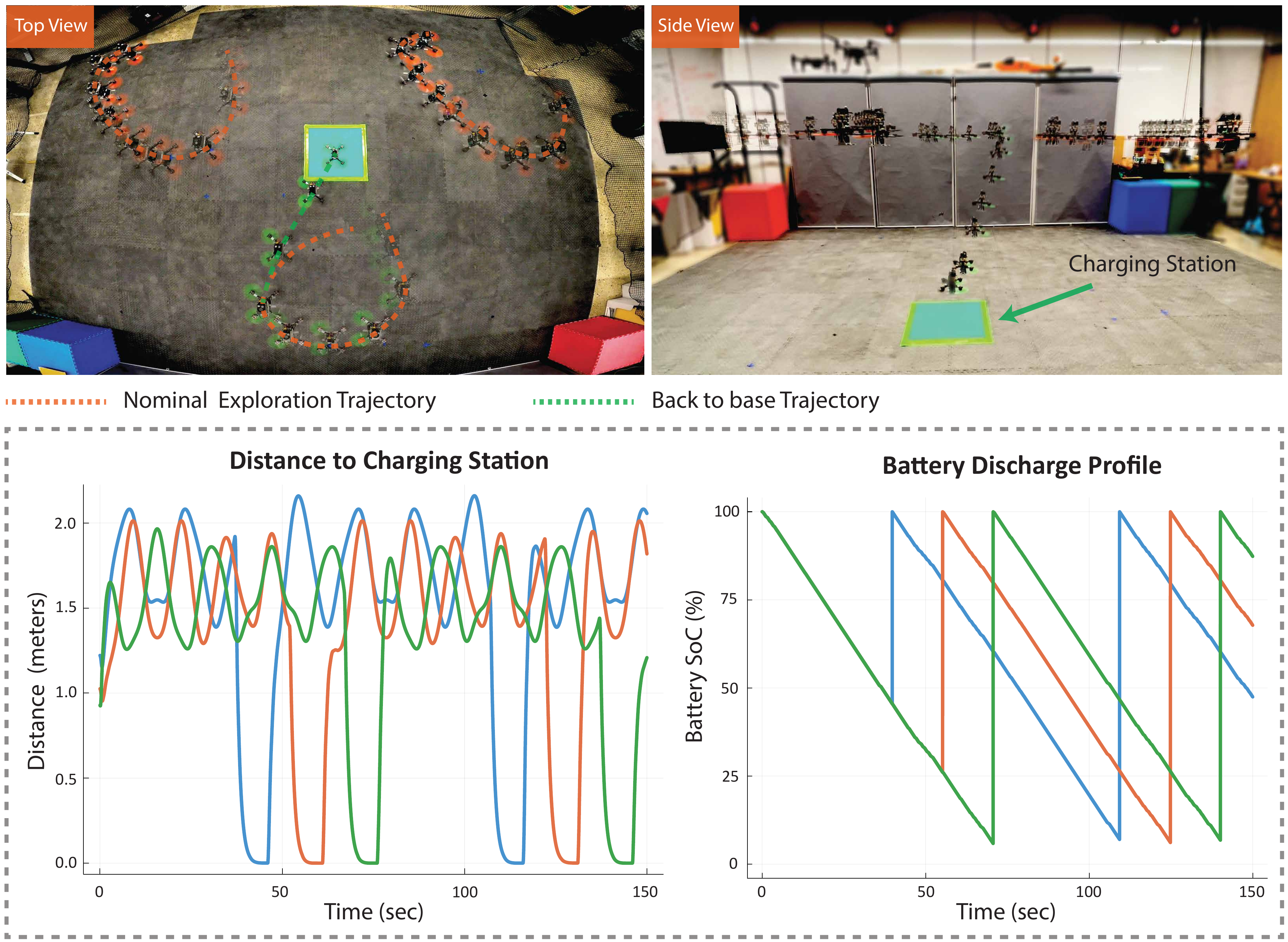}
  \caption{The sequence of three quadrotor trajectories is shown, with one returning to the charging station. The plots display the distance to the station and the battery discharge profile over 150 seconds.}
  \vspace{-20pt}
  \label{fig:mesch_static_Exp}
\end{figure}
\subsubsection{\textbf{Computational efficiency and scalability}}


Distributing computation across the robot network improves the efficiency of the \mesch{} module. The main overhead comes from generating candidate trajectories, with solving \eqref{eq:b2b_prob_overall} and integrating the system’s nonlinear dynamics taking 150 ms and 30 ms on average, respectively. We employ the communication architecture(s) shown in \cref{fig:comms}.

To support real-time applications, each rechargeable robot (e.g. Quad 1) generates candidate trajectories on board, which are transmitted to the central node (Base) for scheduling. The scheduling algorithm has time complexity $\mathcal{O}(N \log N)$, mainly due to the sorting function in line 2 of \cref{alg:gware}. Thus, the method scales with $\mathcal{O}(N \log N)$, where $N$ is the number of rechargeable robots. To demonstrate scalability, we evaluate the method with 40 rechargeable quadrotors (shown in the attached video). In the static case, quadrotors return with $(3\pm1) \%$ battery SoC remaining. With the mobile charging robot, they return with $8\pm1.5\%$ SoC, accounting for reserve energy.

\subsubsection{\textbf{Feasibility and Robustness of \mesch{}}}
If the Gap flag \eqref{gap_flag} and Reserve SoC \eqref{min_soc} conditions hold in the first iteration, \mesch{} guarantees the satisfaction of \eqref{eq:min_energy_constraints} and \eqref{eq:gap_lower_bound} for all times, assuming \eqref{eq:landing_overall} and \eqref{eq:b2b_prob_overall} are solvable as stated in \cref{thm1}. \mesch{} is robust to rechargeable robot failures, as it reduces the number of gap flag conditions \eqref{gap_flag} to check, like removing an inequality constraint in optimization. However, it is not provably robust to new robots, as they might add potentially infeasible constraints. Additionally, \mesch{} is vulnerable to central node failure, which could violate the minimum desired gap due to its centralized nature.

\subsubsection{\textbf{Hardware Demonstration}} 
We validate \mesch{} in two setups: (1) three quadrotors with a static charging station, and (2) two quadrotors with a rover as a mobile charging station. Each quadrotor uses an NVIDIA Orin NX for onboard computing, and the rover has a Raspberry Pi. The communication architecture is shown in \cref{fig:comms}. Experiments were conducted in an indoor arena with 17 Vicon cameras for state estimation. Snapshots and plots are in \cref{fig:mesch_static_Exp} and \cref{fig:mesch_mobile_Exp}. In both setups, quadrotors follow nominal Lissajous coverage trajectories, with only $T_N = 2.0$ s of the future nominal trajectory available at each iteration of \mesch{}.
\begin{figure} [t]
  \centering
  \includegraphics[width=1.0\columnwidth]{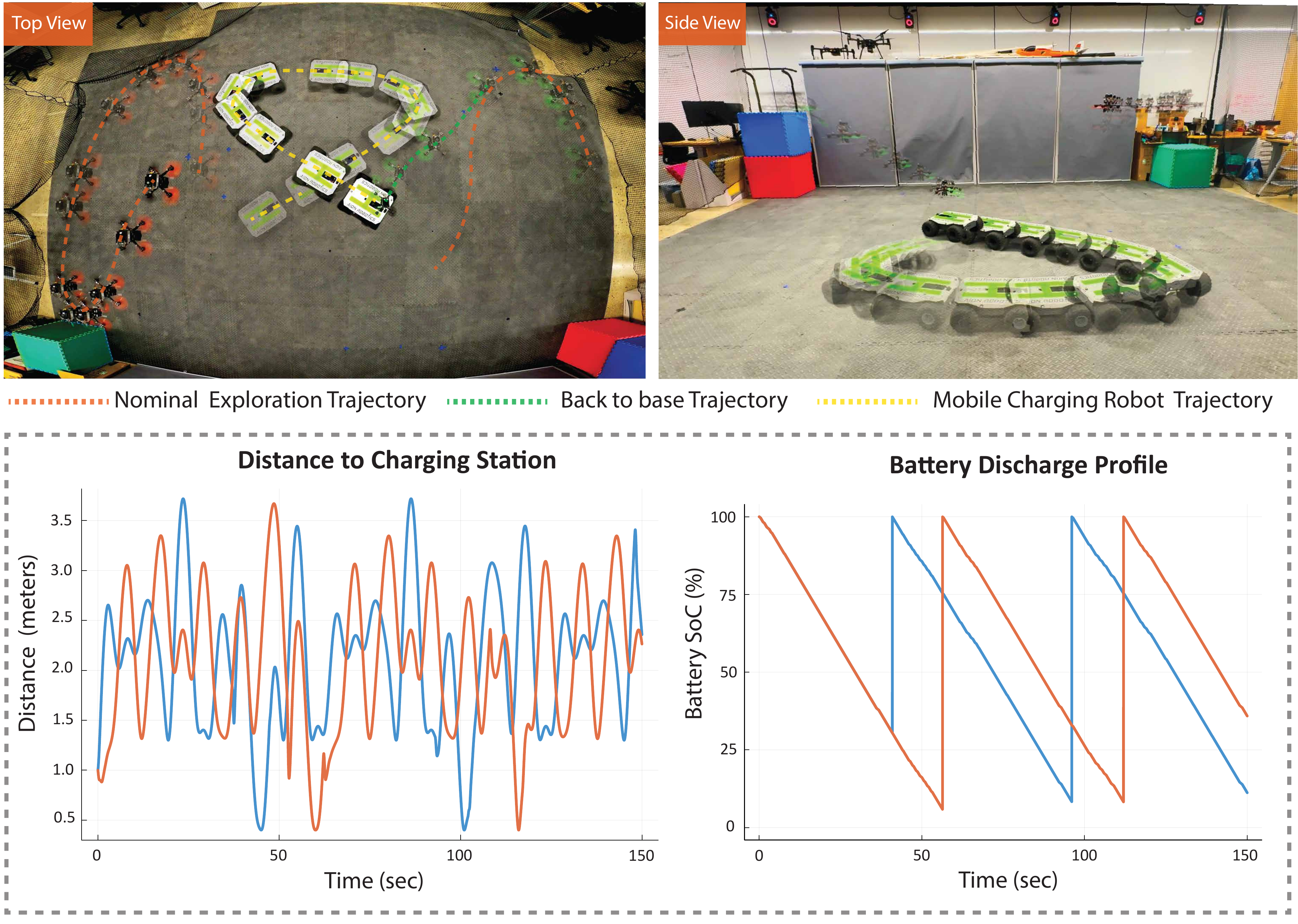}
  \caption{The trajectories of two quadrotors and the rover over a short horizon are shown, with one quadrotor returning to the rover for recharging. The plots display the distance to the charging station and the battery discharge profile over the full 150-second mission.}
  \vspace{-20pt}
  \label{fig:mesch_mobile_Exp}
\end{figure}
Candidate trajectories are generated onboard the quadrotors and sent to the rover’s (or base) computer to verify gap flags and minimum state-of-charge (SoC) conditions. The rover continuously publishes its state and nominal trajectory to assist in trajectory generation.  We also account for delays caused by computational overhead and ROS2 message latency in our implementation. The primary computational costs include candidate trajectory generation and forward propagation ($T_1$), gap flag construction and checking ($T_2$), and miscellaneous delays due to message latency in ROS2 ($T_3$). As long as $T_1 + T_2 + T_3 < T_E$, where $T_E$ is the \mesch{} decision interval, the algorithm ensures recursive feasibility as discussed in \cref{thm1}. In our experiments with three quadrotors, we observed a latency of $600 \pm 150 \text{ ms}$, while $T_E$ was set to $1.5$ s. 

We release all simulation and experimental code. \mesch{} is available as a Julia module, functioning as a low-level filter for a planner. We also provide a Python-ROS2 wrapper for Julia, along with a Docker container for the experiments, and our in-house-built Orin-based DevQuad \cite{gatekeeper}.

\section{Conclusion and Future Work}
In this work, we present \mesch{}, a framework for managing the exclusive use of a mobile charging station by robots performing persistent tasks. Through analysis and case studies, we show that \mesch{} efficiently schedules recharges for robots with nonlinear dynamics, varying battery capacities, and discharge rates, while accounting for uncertainty in the charging station’s position. Future work will enhance \mesch{}'s robustness to central node failure via a distributed implementation, demonstrate its adaptability with continuously replanned trajectories, and refine candidate trajectory generation—especially b2b trajectories—by optimizing them for mission-specific goals like information maximization.

\nocite{*}
\bibliographystyle{IEEEtran}
\bibliography{biblio}
\appendix
\subsection{Proof of \cref{thm1}}
\label{proof_theorem_1}
\begin{proof} The proof, inspired by \cite[Thm. 1]{gatekeeper}, uses induction. 
\subsubsection*{Base Case} At the time $t_1$ and iteration $j = 1$, since both Gap flag condition \eqref{gap_flag} and the Reserve SoC condition \eqref{min_soc} are true, the candidate trajectories are committed for all rechargeable robots i.e. $\forall i \in \Rcal$ and $\forall t^{i_1}_{m_1},t^{i_2}_{m_2}  \in \Tcal$
\begin{subequations}
    \eqnN{
&x_{1}^{i, com}(t) \gets x_{1}^{i, can}(t)  \quad \forall t \in [t_1, t^i_{1,C})\\
&\implies \begin{cases}
    T_{F,1}^{k} > (T_R + T_{E} + kT_{\delta}) \quad \forall k \in \Rcal'\\
    e^i(t) > e^{res}_{1} \quad \forall t \in [t_{1}, t^i_{1, C})
\end{cases} \\
&\implies \begin{cases}
    |t^{i_1}_{m_1} - t^{i_2}_{m_2}| > T_{\delta}  &\forall t \in [t_{1}, t_{1, R} )   \\
    e^i(t) > e_{min}^i \quad &\forall t \in [t_{1}, t_{1, R} )
\end{cases} 
}
\end{subequations}

Since $t_{1, R} > t_{0, R} > t^i_{0, C} \forall i \in \Rcal$, the claim holds.
\subsubsection*{Induction Step} Suppose the claim is true for some $j \in \posintegers$. We show that the claim is true for $j + 1$. 

\subsubsection*{Case 1} When candidate trajectories for all rechargeable robots are valid, i.e. $\forall i \in \Rcal$ and and $\forall t^{i_1}_{m_1},t^{i_2}_{m_2}  \in \Tcal$
\begin{subequations}
    \eqnN{
&x_{j+1}^{i, com}(t) \gets x_{j+1}^{i, can}(t) \quad \forall t \in [t_{j+1}, t^i_{j+1,C})\\
&\implies \begin{cases}
    T_{F,j+1}^{k} > (T_R + T_{E} + kT_{\delta}) \quad \forall k \in \Rcal'\\
    e^i(t) > e^{res}_{j+1} \quad \forall t \in [t_{j+1}, t^i_{j+1, C})
\end{cases} \\
&\implies \begin{cases}
    |t^{i_1}_{m_1} - t^{i_2}_{m_2}| > T_{\delta}  &\forall t \in [t_{j+1}, t_{j+1, R} )  \\
    e^i(t) > e_{min}^i \quad &\forall t \in [t_{j+1}, t_{j+1, R} )
\end{cases} 
}
\end{subequations}
Since $ t_{j+1, R} >  t_{j, R}, \forall i \in \Rcal$  the claim holds.
\subsubsection*{Case 2} This case corresponds to the scenario when the $0^{'th}$ robot in $\Rcal'$ returns either due to violation of Gap flag condition or the Reserve SoC condition, i.e.,
\eqnN{
x_{j+1}^{0', com}(t) \gets x_{j}^{0', com}(t) \quad \forall t \in [t_{j+1}, t^{0'}_{j,C}).}
The candidate trajectories are committed for the remaining robots, i.e. $\forall k \in \Rcal^'\backslash\{0'\}$ and $\forall t^{i_1}_{m_1},t^{i_2}_{m_2}  \in \Tcal$
\begin{subequations}
    \eqnN{
&x_{j+1}^{k, com} \gets x_{j+1}^{k, can} \quad  \forall t \in [t_{j+1}, t^k_{j+1,C})\\
&\implies \begin{cases}
    T_{F,j+1}^{k} > (T_R + T_{E} + kT_{\delta}) \\
    e^k(t) > e^{res}_{j+1} \quad \forall t \in [t_{j+1}, t^k_{j+1, C})
\end{cases} \\
&\implies \begin{cases}
    |t^{i_1}_{m_1} - t^{i_2}_{m_2}| > T_{\delta}  &\forall t \in [t_{j+1}, t_{j+1, R} )  \\
    e^k(t) > e_{min}^k \quad &\forall t \in [t_{j+1}, t_{j+1, R} )
\end{cases} 
}
\end{subequations}

Since $t_{j+1, R} > t^k_{j, C}$,  the claim holds.  
\end{proof}

\end{document}